\documentclass[conference]{IEEEtran}
\IEEEoverridecommandlockouts
\usepackage{cite}
\usepackage{algorithmic}
\usepackage{graphicx}
\usepackage{textcomp}
\usepackage{xcolor}

\usepackage{hyperref}
\usepackage[misc]{ifsym}

\usepackage{graphicx}

\usepackage{bbm}
\usepackage{amssymb}
\usepackage{amsmath}
\usepackage{amsfonts}       
\usepackage{amsthm} 
\usepackage{times}
\usepackage{epsfig}
\usepackage{float}
\usepackage{booktabs}
\usepackage{array}
\usepackage[ruled,vlined,titlenumbered,linesnumbered]{algorithm2e}

\usepackage{url}            

\newtheorem{lem}{Lemma}
\newtheorem{propo}{Proposition}

\newcommand{\RR}{\ensuremath{\mathbb{R}}}

\newcommand{\HH}{\ensuremath{\mathbb{R}^d}}
\newcommand{\LL}{\ensuremath{\mathbb{R}^k}}

\newcommand{\calB}{\mathcal{B}}

\newcommand{\normzero}[1]{{\|{#1}\|}_{0}}
\newcommand{\normone}[1]{{\|{#1}\|}_{1}}

\newcommand{\normOneInf}[1]{{\|{#1}\|}_{1,\infty}}

\newcommand{\orc}{\includegraphics[height=\fontcharht\font`A]{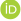}}

\begin{document}

\title{Near-Linear Time Projection onto the $\ell_{1,\infty}$ Ball;\\
Application to Sparse Neural Networks
}

\author{\IEEEauthorblockN{Guillaume Perez\href{mailto:guillaume.perez06@gmail.com}{\Letter}
\href{https://orcid.org/0000-0001-6473-583X}{\orc}}
\IEEEauthorblockA{\textit{Universit\'e C\^ote d'Azur, CNRS} \\
Sophia Antipolis, 06900, France}
\and
\IEEEauthorblockN{Laurent Condat\href{mailto:laurent.condat@kaust.edu.sa}{\Letter}
\href{https://orcid.org/0000-0001-7087-1002}{\orc}}
\IEEEauthorblockA{\textit{King Abdullah University of Science and Technology} \\
 \textit{(KAUST)}, Thuwal, Kingdom of Saudi Arabia\\
 \& SDAIA-KAUST AI}
\and
\IEEEauthorblockN{Michel Barlaud\href{mailto:barlaud@i3s.unice.fr}{\Letter}
\href{https://orcid.org/0000-0001-9093-033X}{\orc}}
\IEEEauthorblockA{\textit{Universit\'e C\^ote d'Azur, CNRS} \\
Sophia Antipolis, 06900, France}
}


\maketitle

\begin{abstract}
Looking for sparsity is nowadays crucial to speed up the training of large-scale neural networks.
Projections onto the $\ell_{1}$ and $\ell_{1,\infty}$ are among the 
most efficient techniques to sparsify and reduce the overall cost of neural networks.
In this paper, we introduce a new projection algorithm for the $\ell_{1,\infty}$ norm ball.
Its worst-case time complexity  is $\mathcal{O}\big(nm+J\log(nm)\big)$ for a matrix in $\mathbb{R}^{n\times m}$.
$J$ is a term that tends to 0 when the sparsity is high, and to $ n \times m$ in the worst case.
The algorithm is easy to implement and it is guaranteed to converge to the exact solution in finite time.
Moreover, we propose to incorporate the $\ell_{1,\infty}$ ball projection while training an autoencoder
to enforce feature selection and sparsity of the weights.
Sparsification appears in the encoder to primarily do feature selection due to our application in biology,
where only a very small part ($<2\%$) of the data is relevant.
We show that in both the biological and general cases of sparsity, our method is the fastest.
\end{abstract}


\section{Introduction}
It is well known that the impressive performance of neural networks is achieved at the cost of a high processing complexity and large memory usage.
In fact, energy consumption and memory limits are the main bottleneck for training neural networks \cite{zeng2021pangu,AIgreen}.
This implies that most of the manpower energy is put into making the current hardware architectures able to work with such a high demand.
Such methods range from parallelism to rematerialization \cite{kumar2019efficient,jain2020checkmate}, 
the latter being NP-hard to solve.
Recently, advances in sparse recovery and deep learning have shown that training neural networks with 
sparse weights not only improves the processing time and batch sizes,
but most importantly improves the robustness and test accuracy of the learned models.

Looking for sparsity appears in many machine learning applications,
such as the identification of biomarkers in biology \cite{abeel2009robust,he2010stable} 
or the recovery of sparse signals in compressed sensing \cite{donoho2006compressed,wright09,figueiredo07}.
For example, consider the problem of minimizing a reconstruction cost function $F$ of a parameter vector $x$. 
In addition, consider constraining the number of nonzero components ($\ell_0$ seminorm) of the learned vector to at most a given sparsity value:
\begin{equation*}
\underset{x \in \RR^d}{\text{minimize}} \quad F(x) \quad \text{ subject to }  \quad \normzero{x} \leq \epsilon.
\end{equation*}
This problem is called \textit{feature selection} and has been a large research area in machine learning.
Unfortunately, this problem is generally strictly nonconvex, combinatorial, and very difficult to solve \cite{natarajan1995sparse}. 
Nevertheless, many relaxed methods have been proposed, such as the \textit{LASSO} method \cite{tRS,sparse}, 
which considers the  $\ell_1$ norm instead of the  $\ell_0$ seminorm of $x$.
One of the reasons why such regularization techniques are widely used is that 
Cand\`es and Tao proved that using a  $\ell_1$ projection gives near-optimal guarantees on the reconstruction loss \cite{candes2006near}. 
Since then, many methods have been defined using either the $\ell_1$ or the reweighed $\ell_1$ norm for sparse regularization \cite{candes2008enhancing}.

Solving such optimization problems is usually done using projected gradient descent (GD).
Given the current point $x_t$ and the objective function $F$ to optimize, a GD step is taken toward the objective:
$x_{t+1} = x_t-\gamma\nabla F(x)$, for some stepsize, or learning rate, $\gamma>0$.
GD does not take into account the presence of  constraints, 
hence constraints are usually inserted in the objective using Lagrange multipliers, or using projection or proximal methods.
The projected GD algorithm is then $x_{t+1} = \alpha P_C\big(x_t-\gamma\nabla F(x)\big) + (1-\alpha)x_t$, with $P_C$ the projection or proximal operator.
Note that projecting onto the $\ell_1$ or reweighed $\ell_1$ norm ball is of linear-time complexity and is the common choice \cite{Perez19,perez22}.

In deep learning, exploiting the sparsity of neural networks has been a long-lasting topic.
\textit{Dropout} for instance is an early implementation of sparsity, whose goal
is to increase the robustness of the learned representation \cite{dropout1,dropout2,ouyang2022block}.
While dropout drastically improves the robustness of non-sparse neural networks, 
feature selection methods have proved more efficient to find robust and sparse models, leading to better accuracy.
Indeed, in recent years, numerous methods have been proposed in order to \textit{sparsify} the weights during the training phase \cite{Tar2018,Zho2016}.
For example, sparse iso-flops or similar methods aim at replacing dense layers with transformation to improve the representation capacity \cite{saxena2023sift,ma2021effective}.
Other methods generally do produce sparse weight matrices, but this sparsity, while helping the accuracy, was not memory or processing efficient. 
To address this issue, the group-LASSO Lagrangian approach  was proposed \cite{Yua}, in order to directly sparsify neurons without loss of performance \cite{Alv2016,Hua2018,Osw2016}. 
For every $p,q \in \RR$, the 
$\ell_{p,q}$
norm of a real matrix 
$X = [x_1\ \cdots\ x_m] \in \RR^{n\times m}$ with columns $x_j$ and elements $X_{i,j}$ 
is given by
\begin{equation}
  \|X\|_{p,q} := \bigg(\sum_{j=1}^m \|x_j\|_q^p\bigg)^{\frac{1}{p}},
\end{equation}
where the $\ell_q$ norm of the vector $x_j\in\mathbb{R}^n$ is 
\begin{equation} 
  \|x_i\|_{q} := \bigg(\sum_{i=1}^n |X_{i,j}|^q\bigg)^{\frac{1}{q}}.
\end{equation}
By extension, the $\ell_\infty$ norm of $x_j$ is
\begin{equation} 
  \|x_j\|_{\infty} := \max_{i=1,\ldots,n} |X_{i,j}|.
\end{equation}

The $\ell_1$ ball projection and its derivatives have been used 
to enforce sparsity everywhere in deep neural networks, including 
from fully-connected layers to self-attention layers \cite{cui2019fine}, 
and even as a replacement for the softmax activation \cite{laha2018controllable}.
Thus, more efficient projection algorithms 
have the potential to impact a large part of the deep-learning community.

The $\ell_{1,\infty}$ norm is 
of particular interest because, compared to other norms, it is able to set a whole set of columns to zero, 
instead of spreading zeros as done by the $\ell_1$ norm.
This makes it particularly interesting for machine learning applications, 
and this is why many projection algorithms have been proposed
\cite{quattoni2009efficient,chau2019efficient,chu2020semismooth,bejar2021fastest}.

In this paper, we introduce a new projection algorithm for the $\ell_{1,\infty}$ norm ball.
The worst-case time complexity of this algorithm is $\mathcal{O}\big(nm+J\log(nm)\big)$ for a matrix in $\RR^{n\times m}$.
$J$ is a term that tends to 0 when the sparsity is high, and to $ n \times m$ in the worst case.
Moreover, as shown in our experimental section, for some matrices,
when the sparsity hits 0 (i.e. no sparsity), the $J$ value is around 3\% of the matrix size, which implies an almost linear complexity $\mathcal{O}\big(nm\big)$.
While recent algorithms are either approximate or based on complex reformulations, like semismooth Newton-type methods,
the proposed algorithm is simple yet efficient.
As shown in the experimental section, it is faster than all existing algorithms in the presence of sparsity.

Moreover, we propose to incorporate the $\ell_{1,\infty}$ ball projection while training an autoencoder
to enforce feature selection and sparsity of the weights.
Sparsification appears in the encoder to primarily do feature selection due to our application in biology,
where only a very small part ($<2\%$) of the data is relevant.
As shown in our experimental section, this setting allows us to accurately extract a tiny set (around 50) of relevant features from around three thousand biomarkers.
Our experimental section is split in two parts.
First, we provide an empirical analysis of the projection algorithms onto the  $\ell_{1,\infty}$ ball.
This part shows the advantage of the proposed method, especially in the context of sparsity. Second, we apply our framework on two biological datasets. In biology, the number of features (RNA or proteins) is very
large. To make a diagnosis, only a reduced number
of features is required. The problem is to select informative features. We show the advantage of using the  $\ell_{1,\infty}$ norm as a regularizer instead of other projection methods.

\section{$\ell_{1,\infty}$ ball, simplex, and Projection}
The projection onto the $\ell_{1,\infty}$ ball has gained interest in the last years 
\cite{quattoni2009efficient,chau2019efficient,chu2020semismooth,bejar2021fastest}, due to 
its efficiency to enforce sparsity and,  most importantly, to often increase accuracy. 
In this section, we formulate the problem
and derive a near-linear algorithm for efficient sparse projection.

\subsection{Definitions}
Let $Y \in \RR^{n \times m}$ be a real matrix of dimensions $m\geq 1$, $n\geq 1$, 
with elements 
$Y_{i,j}$, $i=1,\ldots,n$, $j=1,\ldots,m$. 
The $\ell_{1,\infty}$ norm of $Y $ is:
\begin{equation}
    \normOneInf{Y} := \sum_{j=1}^m \max_{i=1,\ldots,n} |Y_{i,j}|.
\end{equation}
Given a radius $C\geq 0$, the goal is to project $Y$ onto the $\ell_{1,\infty}$ norm ball of radius $C$, denoted by:
\begin{equation}
\calB^C_{1,\infty}:=\left\{X\in\RR^{n \times m}\ : \ \normOneInf{X}\leq C\right\}.
\end{equation}
The projection $P_{\calB^C_{1,\infty}}$ onto $\calB^C_{1,\infty}$ is given by:
\begin{equation}
P_{\calB^C_{1,\infty}}:Y\mapsto \arg \min \limits_{X \in \calB^C_{1,\infty}}   \frac{1}{2} \|X-Y\|_\mathrm{F}^2, 
\end{equation}
where $\|\cdot\|_\mathrm{F}=\|\cdot\|_{2,2}$ is the Frobenius norm. 
This projection can be derived from the projection onto the solid simplex $\Delta^C_{1,\infty}$:
\begin{equation}
\Delta^C_{1,\infty}:=\left\{X\in\RR_+^{n \times m}\, \ :\ \normOneInf{X}\leq C \;\right\}\!,
\end{equation}
where $\RR_+$ is the set of nonnegative reals. 
Indeed, let the sign function be defined by sign$(x) := \{-1 $ if $ x<0; 0 $ if $ x = 0; 1 $ if $ x > 0\}$.
The projection of  $Y \in \RR^{n\times m}$ onto  $\calB^C_{1,\infty}$ is given by:
\begin{equation}
P_{\calB^C_{1,\infty}}(Y)=
 \text{sign}(Y) \odot P_{\Delta^C_{1,\infty}}(|Y|),	
\end{equation}
with $\odot$ the Hadamard, or elementwise, product and $|Y|$ the elementwise absolute value of $Y$.
Moreover, if $\normOneInf{Y}\leq C$, $P_{\calB^C_{1,\infty}}(Y)=Y$. Thus, in the following, we focus on the projection onto $\Delta^C_{1,\infty}$ of a matrix $Y$ with $\normOneInf{Y}> C$ and nonnegative elements.
This projection 
can be characterized using auxiliary variables $\mu_j$, $j=1,\ldots,m$,  as: 
\begin{eqnarray}\label{eq:prob:P}
P_{\Delta^C_{1,\infty}}
:&Y\mapsto\arg \min \limits_{X,\mu} &  \frac{1}{2} \sum_{i,j}(X_{i,j}-Y_{i,j})^2 \\
&\text{subject to }&  \forall i,j, \quad X_{i,j} \leq \mu_j  \label{eq:P1inf:leqmu}\\
& & \sum_{j=1}^m \mu_j = C                                        \label{eq:sum:C}\\
& & \forall i,j, \quad X_{i,j} \geq 0.   \label{eq:P1inf:lasteq}         
\end{eqnarray}

\subsection{Properties}

In the above reformulation, the objective is a direct expression of the squared distance.
The constraint (\ref{eq:P1inf:leqmu}) enforces an upper bound on the values of the $j$-th column of $X$.
The constraint  (\ref{eq:sum:C}) enforces that the sum of the maximum values is equal to the radius $C$.
The last constraint ensures non-negativity.
The Lagrangian objective of this problem is:
\begin{eqnarray*}
\mathcal{L}_{1,\infty}&:=&
\frac{1}{2} \sum_{i,j}(X_{i,j}-Y_{i,j})^2 + \sum_{i,j} \alpha_{i,j}(X_{i,j}-\mu_i)  \\
& &    + \theta(\sum_i \mu_i - C)- \sum_{i,j} \beta_{i,j}X_{i,j}. 
\end{eqnarray*}

\begin{lem}\label{lem:opt}
At the optimal solution of problem 
\eqref{eq:prob:P}--\eqref{eq:P1inf:lasteq}, 
there exists a constant $\theta \geq 0$ such that for every $j=1,\ldots,m$: either $\mu_j > 0$ and $\sum_i(Y_{i,j}-X_{i,j})=\theta$; or $\mu_j = 0$, $\sum_i Y_{i,j} \leq \theta$, and $\forall i=1,\ldots,n$, $X_{i,j}=0$.
\end{lem}
The proof is given in \cite{quattoni2009efficient} and is
a direct application of the Kuhn--Tucker theorem \cite{hanson1981sufficiency}.
This lemma shows that a quantity $\theta$ is removed from all the columns of the matrix
whose sum is greater than $\theta$, otherwise the whole column is set to zero.

Let $P_{\Delta^\theta_{1}}$ be the projection onto 
$\Delta^\theta_{1}:=\left\{x\in\RR_+^{n}\, \ :\ \sum_{i=1}^n x_n\leq \theta \;\right\}$, 
the solid simplex of radius $\theta$.
\begin{propo}
Let $Y = [y_1\ \cdots\ y_m] \in \RR_+^{n,m}$ be a matrix such that $\normOneInf{Y}> C$. 
Then
\begin{equation} \label{eq:reform}
    P_{\Delta^C_{1,\infty}}(Y) = \big[y_1 - P_{\Delta^\theta_{1}}(y_1)\ \cdots\ y_m - P_{\Delta^\theta_{1}}(y_m)\big],
\end{equation}
with $\theta$ defined in Lemma~\ref{lem:opt}.
\end{propo}
\begin{proof} 
Consider a column $y_j$ whose sum of elements is less than or equal to $\theta$. Then, 
$y_j = P_{\Delta^\theta_{1}}(y_j)$ so that
$y_j - P_{\Delta^\theta_{1}}(y_j)$ is the zero vector.
Now consider a column $y_j$ whose sum of elements is greater than $\theta$. We have, for every $i=1,\ldots,n$, $X_{i,j}=\min(Y_{i,j},\mu_j)$. Also, by properties of the 
projection onto $\Delta^\theta_{1}$,
$z_j:=P_{\Delta^\theta_{1}}(y_j)$ satisfies \cite{condat,Perez19}, for every $i=1,\ldots,n$, $Z_{i,j}=\max(Y_{i,j}-\mu_j,0)$, so that $X_{i,j}=\min(Y_{i,j},\mu_j)=Y_{i,j}-Z_{i,j}$. 
Hence, $x_j=y_j-z_j$. Also, $\sum_i(Y_{i,j}-X_{i,j})=\sum_i(Y_{i,j} - \max(Y_{i,j}-\mu_j,0))) = \sum_i(\max(Y_{i,j}-\mu_i,0)) 
= \sum_i Z_{i,j}=
\theta$.
\end{proof}

Thus, if $\theta$ was known, the projection onto $\Delta^C_{1,\infty}$ would be easily done using $m$ projections onto $\Delta^\theta_{1}$. Thus, the difficulty  essentially lies in finding $\theta$.

\subsection{Relation between the $\ell_{1,\infty}$ and $\ell_{\infty,1}$ norms}

As detailed in Section 2 of \cite{bejar2021fastest}, the projection onto the $\ell_{1,\infty}$ norm ball can be used to compute the proximity operator of the dual norm, which is the $\ell_{\infty,1}$ norm: 
\begin{equation}
\|Y\|_{\infty,1}:=\max_{j=1,\ldots,m}
\sum_{i=1}^n  |Y_{i,j}|.
\end{equation}
Given a matrix $Y\in\mathbb{R}^{n\times m}$ and a regularization parameter $C>0$, the proximity operator of $C\|\cdot\|_{\infty,1}$ is the mapping
\cite{moreau62}
\begin{equation}
\mathrm{prox}_{C\|\cdot\|_{\infty,1}}:Y\mapsto \arg \min \limits_{X\in\mathbb{R}^{n\times m}}   \frac{1}{2} \|X-Y\|_\mathrm{F}^2 + C\|X\|_{\infty,1}.\label{eqprox}
\end{equation}
Thus, computing this proximity operator amounts to solving the optimization problem in \eqref{eqprox}.
This operator can be used as a subroutine in proximal splitting algorithms \cite{con23} to solve more complicated problems involving the $\ell_{\infty,1}$ norm.

Then, by virtue of the Moreau identity \cite{bau17}, computing this proximity operator is equivalent to projecting onto the $\ell_{1,\infty}$ norm ball:
\begin{equation}
\mathrm{prox}_{C\|\cdot\|_{\infty,1}}(Y)=Y - P_{\calB^C_{1,\infty}}(Y).
\end{equation}
Hence, our projection algorithm can also be used in problems involving the $\ell_{\infty,1}$ norm.

\section{Projection algorithms}

\subsection{Algorithmic mechanisms}
Let $Y_{j}^{\mu_j} = \{i:Y_{i,j} \geq \mu_j\}$ the set of locations from column $j$ where the values are greater than $\mu_j$.
From the definition of the $\ell_1$ simplex we can extract:

\begin{equation} \label{eq:mui}
    \mu_j = \frac{\sum_{i \in Y_{j}^{\mu_j}} Y_{i,j} - \theta}{|Y_{j}^{\mu_j}|},
\end{equation}
with $|Y_{j}^{\mu_j}|$ the cardinality of the set.
Let $\mathbf{a}$ denote the set of active columns ($a_j = 1 \implies \mu_j>0$).
Let $A = \{i,...,j\}$ the set of locations of ones in $\mathbf{a}$.
Using Equation (\ref{eq:mui}) and Equation (\ref{eq:sum:C}) we have

\begin{equation} \label{eq:C}
    C = \frac{\sum_{j \in A} \sum_{i \in Y_{j}^{\mu_j}} Y_{i,j} - \theta}{|Y_{j}^{\mu_j}|}.
\end{equation}
Finally, from equation (\ref{eq:mui}) and equation (\ref{eq:C}), we have
\begin{equation} \label{eq:theta}
    \theta = \frac{\sum_{j\in A} \sum_{i \in Y_{j}^{\mu_j}} \frac{Y_{i,j}}{|Y_{j}^{\mu_j}|} - C}{
    \sum_{j\in A} \frac{1}{|Y_{j}^{\mu_j}|}}.
\end{equation}
Let $Z$ be the matrix where $Z_{i,j}$ is the $i$th greatest value of column $j$ of $Y$.
Let $S$ be the matrix where $S_{i,j}$ is the cumulative sum of the $i$ largest values of column $j$ for $Y$,
$S_{i,j}=\sum_{k=1}^i Z_{k,j}$. 
Let $\theta_t$ be the current approximation of $\theta$.
Consider the addition of an element to $\theta_t$ and its evolution with respect to its previous value.
Let $\theta_{t+1}$ be the new value after another element of $Y$ is added to $\theta_t$.

\begin{propo} \label{propo:addToTheta}
    Adding element $Z_{i+1,j}$ to $\theta_t$ such that $\theta_t >  i Z_{i+1,j} - S_{i,j}$ implies $\theta_{t+1} \geq \theta_t$.
\end{propo}
\begin{propo}\label{propo:removeFromTheta}
    Removing column $j$ from $\theta_t$ if $\sum_i Y_{i,j} \leq \theta_t$ implies $\theta_{t+1} \geq \theta_t$ 
\end{propo}
%

Using these two propositions, whose proofs are in the Appendix, allows to define a first \textit{naive} algorithm.
Algorithm~\ref{alg:ProjC:naive} directly uses  $\ell_1$ projection to perform the $\ell_{1,\infty}$ projection.
This algorithms updates $\theta_t$ until no further modifications are possible.
At line \ref{alg:ProjC:naive:remove} it removes columns with respect to proposition \ref{propo:removeFromTheta}.
At line \ref{alg:ProjC:naive:gather} it gathers all the elements of a column that satisfy proposition \ref{propo:addToTheta}.
This algorithm, despite its simplicity, has been only recently proposed \cite{bejar2021fastest}.
The authors proposed two \textit{efficient} implementations preventing
the recomputation the $\ell_1$ projection from scratch each time.
Nevertheless, its worst-case complexity is $O(n^2m P)$ with $P$ the complexity of projection onto the  $\ell_1$ simplex.

\begin{algorithm}[ht]
\caption{Projection naive \cite{bejar2021fastest}}\label{alg:ProjC:naive}
\KwData{$Y \in \RR^{n,m}_+, C > 0$}
\KwResult{$X = P_{\ell_{1,\infty}}(Y)$}
\SetAlgoLined
$\mathbf{a} \gets $set$(\{1,\dots,m\})$ \\
$\theta \gets \frac{\sum_j \max y_j-c}{m}$ \\
\While{$\theta$ changed}{
    \For{$j \in \mathbf{a}$}{
        \If{$\normone{y_{j}} < \theta$}{ \label{alg:ProjC:naive:remove}
            $\mathbf{a} \gets \mathbf{a}\backslash \{j\}$ \\
            continue
        }
        \textbf{$x_j$} $\gets  P_{1}^\theta(y_j)$ \\
        $S_j \gets \text{set(}\{ i |$\textbf{$x_{i,j}$}$  > 0\})$ \label{alg:ProjC:naive:gather}\\
    }
    $\theta \gets \frac{\sum_{j\in\mathbf{a}} \frac{\sum{i \in S_j} Y_{i,j}}{|S_j|} - C}
    {\sum_{j\in \mathbf{a}}\frac{1}{|S_j|}}$ \\
}
$\forall j ,\mu_j \gets \max(0,\frac{\sum{i \in S_j} Y_{i,j} - \theta}{|S_j|})$ \\
$\forall i,j, X_{i,j} \gets $ min($Y_{i,j},\mu_j$)
\end{algorithm}

\paragraph{Total order}
Proposition \ref{propo:addToTheta} can be used to define a total order of the values of matrix $Y$.
Let $R=\{iZ_{i+1,j}-S_{i,j} |\forall i, \forall j\}$ be the residual matrix of $Y$.
Let $P$ be a non-increasing permutation of $R$.
\begin{lem} \label{lem:totalOrder}
    For all $i,j \in [1,nm]$ such that $i<j$, if $R_{P_i}$ cannot be added to $\theta_t$ with 
respect to proposition \ref{propo:addToTheta}, then $R_{P_j}$ cannot be added too.
\end{lem}
This implies that once $P$ is known, iterating over $P$ until proposition \ref{propo:addToTheta}
can no longer be satisfied is enough to find all the elements that satisfy it.
Here, proposition \ref{propo:removeFromTheta} is ignored, but it can be incorporated into $P$.
Let matrix $R'\in \RR^{n+1,m}$ equal to $R$ for the $n$ first rows.
The additional row filled with $S_{n,j}$ for all $j$.
Let $P'$ be a non-increasing permutation of $R'$.
Lemma \ref{lem:totalOrder} can be directly extended to $P'$.

\paragraph{Build $P'$ then find $\theta$ \cite{quattoni2009efficient}} \label{sec:quattoni}
One of the first published projection algorithms starts by computing $P'$ and then
iterates over the elements of $P'$ until $R'<\theta_t$ \cite{quattoni2009efficient}.
Despite a different presentation, the processing of the residual matrix and its sorting is the same.
Its complexity is $O(nm + nm \log(nm))$, a large part of it being in the preprocessing of $P'$.
The performance of this algorithm is discussed Section \ref{secex1}.

\subsection{Proposed Projection Algorithm}
We propose here to follow a logical path to decrease the time complexity of the total order algorithm \cite{quattoni2009efficient}.
The complexity of computing $Z$ is $O(nm \log(n))$ as each of the columns have to be sorted.
The complexity of computing $P'$ is $O(nm \log(nm))$ as the complete matrix $R'$ has to be sorted.
Then, the final step of finding the first element such that none of the proposition allows to add an element to the computation is linear $O(nm)$ .
More precisely, let $K$ be the index in $P'$ where the algorithm stops.
It corresponds roughly to the number of modified values by the projection, either set to zero, or bounded.
The final step of \cite{quattoni2009efficient} is in fact of complexity $O(K)$,
which implies that the global complexity is $O(nm + nm \log(n) + nm \log(nm) + K)$.
In the next paragraphs, we will decrease the complexity step by step, using algorithmic improvements.
The complete algorithm is then given.

\begin{itemize}
\item \textbf{From $O(nm + nm \log(n) + nm \log(nm) + K)$ to $O(nm + nm\log(n) + K \log(nm))$.}
Projecting vectors onto the $\ell_1$ ball is a well-studied topic \cite{duchi,condat,Perez19}.
One of the first fast algorithms proposed to use a \textit{heap} instead of sorting the complete vector \cite{van2009probing}.
We propose to reuse the same idea.
Given a vector in $\RR^n$, the creation of the heap (i.e. \textit{Heapify}) time complexity is $O(n)$, 
the \textit{Top} operation complexity is $O(1)$, 
the \textit{Pop} operation and \textit{Insert} operation complexity is $O(\log(n))$.
Processing $P'$ requires sorting a vector of size $nm$.
We propose to use a heap to store $P'$ and to extract elements one by one until $\theta$ is found.
As only $K$ iterations over $P'$ are required, the total complexity of this part of the algorithm is 
$O(nm + K\log(nm))$ instead of $O(nm\log(nm) + K)$.
Using a heap for the processing of $P'$ leads to a global worst-case time complexity of 
$O(nm + nm\log(n) + K \log(nm))$.

\item \textbf{From $O(nm + nm\log(n) + K \log(nm))$ to $O(nm + K\log(nm))$.}
At any moment of the algorithm, only the next largest value of a given column might be picked up by $P'$.
This implies that the heap $P'$ can contain only $m$ elements at worst, instead of $nm$ elements.
The counterpart is that each time an element of $P'$ is \textit{popped},
the next greatest value of the column that just got popped must be inserted into the heap.
If $Z$ has been processed, then it is easy to get the next greatest element, but processing $Z$ is costly.
We propose to have one heap per column of $Y$, and each time the next greatest value of the column
is required, then the column's heap is \textit{popped}.
Using a heap for the processing of $P'$ and one heap per column instead of sorting 
leads to a global worst-case complexity of $O(nm + K \log(n) + K \log(m)) = O(nm + K\log(nm))$.

\item \textbf{From $O(nm + K\log(nm)$ to $O(nm + J\log(nm))$.}
The last and most important remark comes from the following point: 
Usually, the projection onto the $\ell_{1,\infty}$ ball is applied to enforce sparsity, 
as in our experimental section where the best accuracy was around 99 percent of sparsity.
In such case, most columns will be zeroed, and many values will be bounded in the remaining columns.
Such a remark implies that $K \approx nm$, which implies that there is almost no gain in complexity
from using all the proposed improvements.
Let $J = nm-K$ be roughly the number of non-modified values of the projected matrix.
As $K$ tends to $nm$, $J$ tends to 0 and vice-versa.

We propose to reverse the iteration over $P'$.
Instead of starting from the beginning and looking for the first value smaller than $\theta$, 
We start from the end of $P'$ and look for the first value greater than $\theta$.
This value is the last value added by proposition \ref{propo:addToTheta} or the last column that need to be removed with respect to proposition \ref{propo:removeFromTheta}.
The worst-case time complexity of this algorithm is $O(nm + J\log(nm))$.
\end{itemize}

\subsection{Implementation} 
A possible implementation is given in Algorithm~\ref{alg:sortOptNP1HeapInv}.
Function UpdateTheta() is $\theta \gets \frac{\sum_j \frac{a_j S_{j}}{k_j} - C}{\sum_j \frac{a_j}{k_j}}$.
First, at line \ref{algo:prop:heapifyG}, the global heap is created.
This heap contains $m$ elements, one for each column.
For each element, two values are given, the first one is the column index, the second one is the sorting key.
The initial sorting key is given by the sum of the elements of a column, this is because we are reversing the total order $P'$.
At line \ref{algo:prop:heapifyR}, if it is the first time that the column is encountered,
it is heapified as it will start being used by the global heap.
Putting the heapify here and not at the beginning is done to spare the time used to heapify the zeroed columns.
Then, the total sum of the elements of the columns is added to the current value of $\theta$.
If the current value of $\theta$ is already dominating the column, then the threshold has been found.
Otherwise, at line \ref{algo:prop:elseElement}, the current element is tested to check if it can be added to the
current approximation of $\theta$.
As shown in our experimental section, this new algorithm is faster 
compared to all other methods for sparse projections, and is the first near-linear method for high sparsity.

\begin{algorithm}[ht!]
\caption{Projection Inverse Total Order}\label{alg:sortOptNP1HeapInv}
\KwData{$Y \in \RR^{n,m}_+, C > 0$}
\KwResult{$X = P_{\ell_{1,\infty}}(Y)$}
\SetAlgoLined
$S \gets (\sum_i y_{i,1},\dots,\sum_i y_{i,m})$ \\
$P \gets $Heapify(($1:-S_1,\dots,m:-S_m)$, global, increasing)\label{algo:prop:heapifyG}\\
$\mathbf{k} \gets ones(m,1) \odot (n+1)$;
$\quad \mathbf{a} \gets zeros(m,1)$;\\
$\theta \gets 0$ \\
\While{$\theta$ changed}{
    \While{NotEmpty($P$)}{ 
        $j \gets$ Top$(P)$; $i \gets k_j$ \\
        $\mathbf{k_j} \gets \mathbf{k_j}-1$\\
        \eIf{$i = n+1$}{ \label{algo:prop:heapifyR}
            $\mathbf{a_j} \gets 1$ ;\quad 
            UpdateTheta() \\
            \If{$\normone{y_{j}} < \theta$}{
                $\mathbf{a_j} \gets 0$ ;\quad 
                UpdateTheta() \\
                Break
            }
            $X_j \gets $Heapify($Y_j$, increasing) \\
        }{ \label{algo:prop:elseElement}
            $S_{j} \gets S_{j} - $Top($X_j$)\\
            UpdateTheta()\\
            \If{$\frac{S_{j} - \theta}{k_j} <$ $Y_{i,j}$}{
                $\mathbf{k_j} \gets \mathbf{k_j}+1$\\
                $S_{j} \gets S_{j} + $Top($X_j$)\\
                UpdateTheta() \\
                Break
            }
        }            
        UpdateTop($P$,$k_j$Top($X_j$) - $S_{j}$); 
        Pop($X_j$)\\ 
        
    }
}
$\forall i,j, X_{i,j} \gets $ min($Y_{i,j}, \max(0,\frac{S_{j} - \theta}{k_j})$)
\end{algorithm}

\paragraph{columns eliminations} 
Performances of \cite{bejar2021fastest} are strongly dependent on a $O(nm+m\log(m))$
preprocess that tries to remove rows that provably will be set to zero.
In the proposed algorithm, there is no need to apply this algorithm as our algorithm ignores such rows by design.
Indeed, as the algorithm works backward, it never reaches rows that are dominated by $\theta$. 
In the worst case, it ends on a dominated row, and will directly discard it.

\section{Projection Experiments}\label{secex1}
This section presents experimental results of the projection operation alone.
The goal of such experiments is to highlight the advantages and drawbacks of the 
proposed and known algorithms.
We compared the proposed method against
\textit{Chu et al.}\cite{chu2020semismooth} which uses a semi-smooth Newton algorithm for the projection.
Then \textit{Quattoni et al.} \cite{quattoni2009efficient}, 
whose algorithm corresponds to the total order defined in section \ref{sec:quattoni}.
Finally, \textit{Bejar et al.} \cite{bejar2021fastest} whose algorithm starts by removing columns that we know will be set to zero,
and then applies Algorithm \ref{alg:ProjC:naive}.
All the code used in this experiment is the code generously provided by the authors of the respective algorithms. 
Only \textit{Chu et al.} and \textit{Bejar et al.} compete in terms of performance against the proposed method. 
All other methods usually take an order of magnitude more times,
hence are not present in most of our figures and tables.
Note that such a result is coherent with already published papers \cite{chu2020semismooth,bejar2021fastest}.
The complete code of these experiments can be found 
online\footnote{\url{https://github.com/memo-p/projection}}.
The code used to implement the proposed method is directly using the standard library of C++ for heaps and vectors.
The experiments were run on an \textit{AMD Ryzen 9 5900X 12-Core Processor 3.70 GHz} desktop machine having 32 GB of memory. 
No parallelism was allowed.

The goal of the projection onto the $\ell_{1,\infty}$ ball is usually to enhance sparsity.
Our first experiment investigates the correlation between the radius $C$ and the induced sparsity, 
and most importantly the running time of the algorithms.
The size of the matrices is 1000x1000, values between 0 and 1 uniformly sampled
and the radius are in $[10^-3,8]$.

\begin{figure}[!h]
    \centering
    \includegraphics[width=0.40\textwidth]{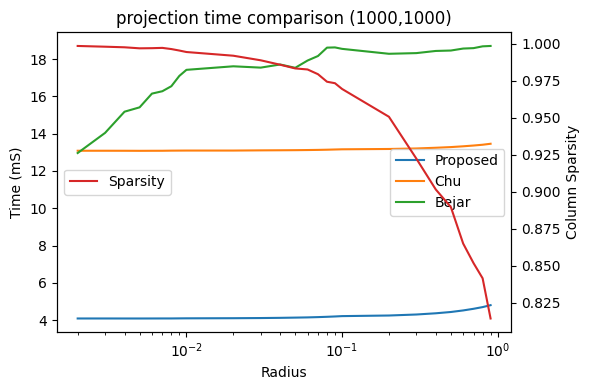}
    \includegraphics[width=0.40\textwidth]{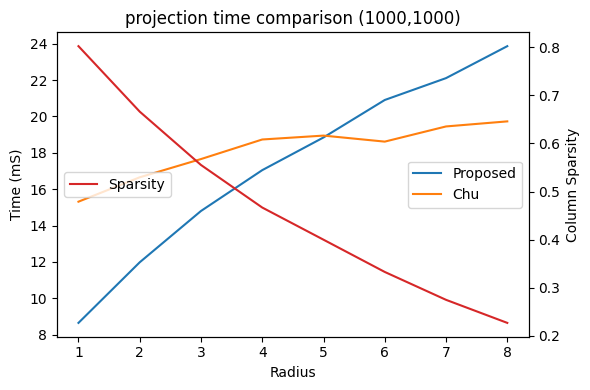}
    \caption{Impact of the radius on the sparsity of the matrix. Comparison of the projection times.}
    \label{fig:xp:Sparsity}
\end{figure}

Figure~\ref{fig:xp:Sparsity} shows that the sparsity decreases exponentially as the radius is increasing.
Moreover, we can see that the proposed algorithm is faster than the best existing methods when the sparsity is at least 40\%.
It is not surprising since the complexity of our method tends to linear when the sparsity is high.
As we can see, when less sparsity is present, the cost of using multiple heaps starts to slow down the algorithm. 
The same kind of results appears when the size of the matrix varies, as shown in Figure \ref{fig:xp:Sparsity2}.

\begin{figure}[!h]
    \centering
    \includegraphics[width=0.40\textwidth]{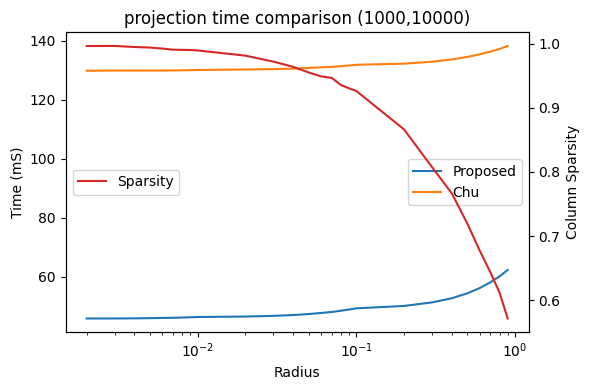}
    \includegraphics[width=0.40\textwidth]{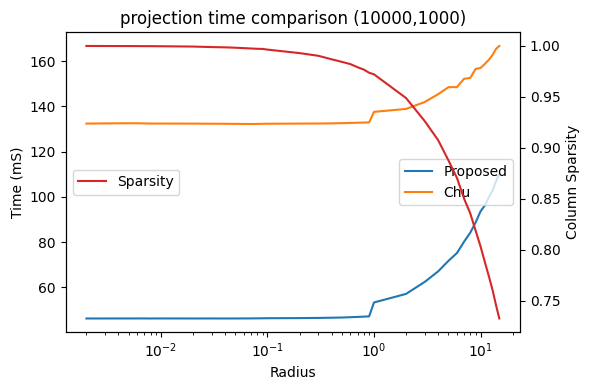}
    \caption{ Projection time for matrix sizes (left) 1000x10000, (right) 10000x1000.}
    \label{fig:xp:Sparsity2}
\end{figure}

For the second experiment, we propose to vary the size of the matrix instead of the radius. 
Figure~\ref{fig:xp:projN} gives a global view of the methods as the matrix size is increasing.
we can see that as the matrix size growth, even for the radius of 1, the proposed method is significantly faster. 
Indeed, we can see that in both cases, the impact of the increase in the size has less impact on the proposed method.
Note that the figure showing increase of size with fixed $n$ is the best scenario for the proposed algorithm as the sparsity is increasing up with the size.
We can see that overall, the proposed method is faster in average than the other methods.
In the CAE experiment of the next section the proposed method was in average 2.2 times faster than \textit{Chu et. al.} given the configuration of the network.

\begin{figure}[t]
    \centering
    \includegraphics[width=0.40\textwidth]{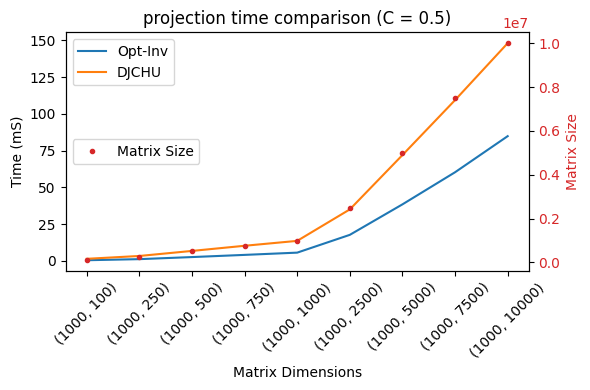}
    \includegraphics[width=0.40\textwidth]{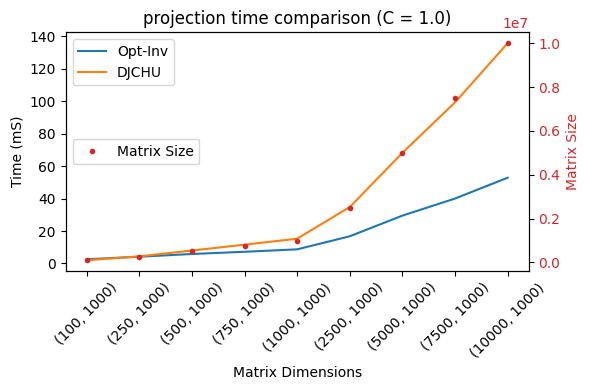}
    \caption{(left) Projection time for a fixed $n$. (right) Projection time for a fixed $m$}
    \label{fig:xp:projN}
\end{figure}

Finally, the complexity of the algorithm ($O(nm + J\log(nm))$) is dependent of a global constant named $J$.
We propose in this experiment to analyze the actual value taken by this constant in the same setting as the previous experiment.
Figure~\ref{fig:xp:J} shows the $J$ value as a percentage of the matrix size. 
Indeed, the range of values of $J$ is $[1,nm]$.
As we can see, the more the sparsity is growing, the smaller $J$ is, with a $J$ next to zero when the sparsity is close to 1.
Moreover, when the sparsity hits $0$, which implies that the resulting matrix is full, the $J$ is slightly above 3\% of the global size of the matrix ($nm$). 
This implies that the algorithm behave almost linearly when even a small amount of sparsity is present.

\begin{figure}[t]
    \centering
    \includegraphics[width=0.40\textwidth]{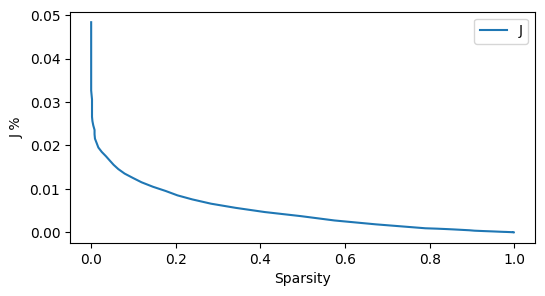}
    \caption{$J$ complexity constant value as a function of the resulting sparsity.
    The \% is of the global size of the matrix ($nm$). }
    \label{fig:xp:J}
\end{figure}

\section{Supervised Autoencoder (SAE) framework}
Autoencoders were introduced within the field of neural networks decades ago, their most efficient application at the time being dimensionality reduction \cite{Hinton_auto,deep}. 
Autoencoders were used in application ranging from unsupervised deep-clustering \cite{guo2017deep} to supervised learning to improve classification performance \cite{VAE,semisupervised,Rochelle}.
In this paper, we use the supervised autoencoder (SAE) neural network, analogously to \cite{ICASSP}, where no constraints as there is no parametric distribution assumption. \\

Let $X$ be the dataset in $\HH$, and $Y$ the labels in $\{0, \dots , k\}$, with $k$ the number of classes.
Let $Z \in \LL$ be the encoded latent vectors, $\widehat{X}$ $\in \HH $ the reconstructed data and $W$ the weights of the neural network.
Note that the dimension of the latent space $k$ corresponds to the number of classes.
Let $E(X)=Z$ be the encoder function of the autoencoder, and let $D(Z)=\widehat{X}$ be the decoder function of the autoencoder.
We use the Cross Entropy Loss as the classification loss $\mathcal{H}$ 
and the robust Smooth $\ell_1$ (Huber) Loss \cite{Huber} as the reconstruction loss $\psi$.
Parameter $\lambda$ is a linear combination factor used to define the final loss
$\phi(X,Y)=$ $\lambda \psi (X,\widehat{X}) +\mathcal{H}(Y,Z)$.

The goal is to learn the network weights $W$ minimizing the total loss.
In order to perform feature selection, as biomedical datasets often present a relatively small number of informative features, we also want to sparsify the network, following the work proposed in \cite{ICASSP}.
We propose to use the $\ell_{1,\infty}$ projection as a constraint to enforce sparsity in our model.
The global problem to minimize is
\begin{equation*}
\underset{W}{\text{minimize}}  \quad \phi(X,Y) \quad \text{ subject to }  \quad \normOneInf{W} \leq C.
\end{equation*}

The double descent algorithm was originally proposed as follows \cite{ICASSP}: after training a network, apply the $\ell_1\ projection$, set all weights smaller than a given threshold to zero, rewind the rest of the weights to their initial configuration, and then retrain the network from this starting configuration while keeping the zero weights frozen (untrained). 
To achieve structured sparsity, we replace the threshold by our $\ell_{1,\infty} $ projection.\\
We train the network using the classical Adam optimizer \cite{Adam}. Note that low values of $ C $ imply high sparsity of the network. 
The impact and selection of such a value is discussed in the next section.

\subsection{SAE experimental results}
We implemented our SAE method using the PyTorch framework for the model, optimizer, schedulers and loss functions. We chose the ADAM optimizer \cite{Adam}, as the standard optimizer in PyTorch.
We used a symmetric linear fully connected network \cite{ICASSP}, with the encoder comprised of an input layer of $d$ neurons, one hidden layer followed by a ReLU or SILU activation function and a latent layer of dimension $k$.\\
We compare the classical $\ell_{1}$  and $\ell_{1,\infty}$ projections. Note that our SAE provides a two-dimensional latent space where the samples can be visualized, and their respective classifications interpreted. 
Finally, our supervised autoencoder specifically provides informative features \cite{Captum} that are particularly insightful for biologists. We provide for each experiment the accuracy and the column sparsity (number of columns set to zero).

To generate artificial biological data to benchmark our $\ell_{1,\infty}$ projection in the SAE framework, we use the $make\_classification$ utility from scikit-learn. This generator controls the separability (set to $ 0.8$) of the synthetic dataset.
We generated $n=1,000$ samples (a number related to the number of samples in large biological datasets) with a  number $d$ of features. We chose  $d=10,000$ as the dimension to test because this is the typical range for biological data. We chose a low number of informative features ($ 64$ ) realistically with biological databases.
Our feature extraction encoder is an FCNN with 1 hidden layer composed of 100 neurons and $k=2$.
The complete code of these experiments (real and synthetic data) can be found 
online.\footnote{\url{https://webcms.i3s.unice.fr/Michel_Barlaud/}}

\begin{figure}[t]
    \centering
    \includegraphics[width=0.24\textwidth]{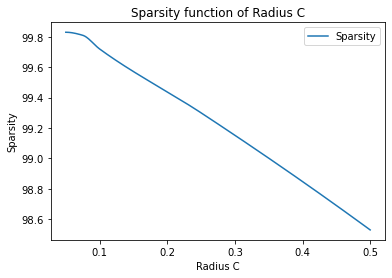}
    \includegraphics[width=0.24\textwidth]{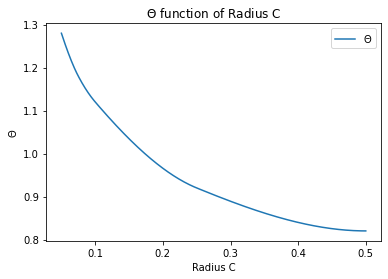}
    \vspace*{-5mm}\caption{Synthetic data. Left: sparsity and parameter $\theta$ as a function of the radius  C.Right:Parameter $\theta$ as a function of the radius  C. }
    \label{Synt_Sparsity}
\end{figure}

Figure~\ref{Synt_Sparsity} (left) shows the impact of the radius on the obtained sparsity.
Unsurprisingly, the larger is the radius, the smaller is the sparsity.
Yet, by considering that the best accuracy is around 0.1, the column sparsity is around $99.6$, 
hence the number of selected features is around 40.
Figure~\ref{Synt_Sparsity} (right) shows the impact of the radius on the obtained parameter $\theta$.
$\theta$ is the threshold used by the projection.
Note that the $\theta$ value does not decrease linearly with respect to the radius.
\begin{figure}[t]
    \centering
    \includegraphics[width=0.49\textwidth]{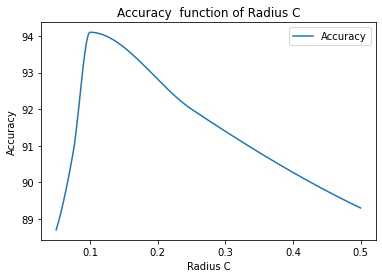}  
    \vspace*{-5mm}\caption{Synthetic data: Accuracy as a function of the radius C. }
    \label{Synt_accuracy}
\end{figure}

Figure \ref{Synt_accuracy} shows accuracy as a function of the radius C. The accuracy is highest for a radius $c=0.1 $.

\begin{table}[!h]
    \centering
    
    \begin{tabular}{|c|c|c|c|c|c|c|c|}
        \hline
          Synthetic data   & Baseline &$\ell_1$    & $\ell_{1,\infty}$   \\
        \hline
        Accuracy $ \%$ &   86.60 $\pm 2.0 $ &   89.1 $\pm 1.8$    & 92.77 $\pm 1.8$  \\
       \hline
        Colsp     &     0  & 81.94  &99.6\\
        \hline
    \end{tabular}\smallskip
    \caption{\textbf{Synthetic} dataset: Metrics over multiple seeds: comparison of no projection and  4 projections methods  $\ell_{1}$ ($\eta=10$), $\ell_{1,\infty}$ (C=0.1), Proj $\ell_{1,\infty}$ (C=0.1) masked.}
    \label{Comp_proj_synt}
\end{table}
Table~\ref{Comp_proj_synt} presents the results of the different possible implementations of the framework.
The baseline is an implementation that does not contain any projection.
It is the usual implementation of neural networks.
Then, $\ell_1$ and $\ell_{1,\infty}$ are the projection of their respective norms.
Compared to the baseline, the SAE using the $\ell_{1,\infty}$ projection improves the accuracy by $6.12 \%$, while using only $0.4\%$ of the features.
Moreover, the  $\ell_{1,\infty}$ projection improved the accuracy obtained with the $\ell_{1}$.
Such a result is not surprising, as the $\ell_{1}$ does not consider the relationship inside columns, and only see the matrix from a global point of view.
Finally, considering now the sparsity,  the  $\ell_{1,\infty}$ projection outperformed the $\ell_{1}$ by $ 15 \%$.

\section{Conclusion and Perspectives}
In this paper we introduced a fast projection algorithm onto the $\ell_{1,\infty}$ ball.
This projection algorithm is exact and of near-linear time complexity when the sparsity is high.
The worst-case time complexity of this algorithm is $\mathcal{O}\big(nm+J\log(nm)\big)$ for a matrix in $\mathbb{R}^{n\times m}$.
$J$ is a term that tends to 0 when the sparsity is high, and to $ n \times m$ in the worst case.
Moreover, as shown in our experimental section, for some matrices,
when the sparsity hits 0, the $J$ value is around 3\% of the matrix size, which implies an almost linear complexity $\mathcal{O}\big(nm\big)$.
Thanks to this complexity, and as shown in our experiments, the proposed algorithm is faster than existing methods.
In addition, the main goal of such a norm is to enforce structured sparsity for neural networks.
As shown in the second part of our experiments, the use of the $\ell_{1,\infty}$ ball to enforce sparsity 
is efficient in terms of feature selection,  accuracy, and  computational complexity.
This result confirms that sparsity-efficient projections should become mainstream for neural network training.
Our future work will involve sparsifying convolutional networks for image coding \cite{Twitter,ICASSP23} and extending the projection with a bilevel approach \cite{BPM}.


\bibliography{references}
\bibliographystyle{elsarticle-num}

\section{Appendix}
Consider the addition of an element to $\theta_t$ and its evolution with respect to its previous value.
Let $\theta_{t+1}$ be the new value after the element $Y_{k,l}$ is added to $\theta_t$.
First, let's consider the impact on its local sum. 
Let $v = \mu_k'$ be the new set of selected values and
$w=\mu_k$ be the value before the addition of the element.

\begin{eqnarray*} \label{eq:moyenne}
    \sum_{j \in Y_{k}^{\mu_k'}} \frac{Y_{k,j}}{|Y_{k}^{\mu_k'}|} &=& \sum_{j \in Y_{k}^{w}} \frac{Y_{k,j}}{|Y_{k}^{\mu_k'}|} + \frac{Y_{k,l}}{|Y_{k}^{\mu_k'}|} \\
    \sum_{j \in Y_{k}^{\mu_k'}} \frac{Y_{k,j}}{|Y_{k}^{\mu_k'}|} &=& 
    \sum_{j \in Y_{k}^{w}} \frac{Y_{k,j}}{|Y_{k}^{w}|} +
    \frac{Y_{k,l}-\overline{Y_i^w}}{|Y_{k}^{\mu_k'}|} \\
\end{eqnarray*}
Then we have:
\begin{eqnarray*} \label{eq:thetaPrime}
    \theta_{t+1} &=& \frac{\sum_{i\in A} 
    \sum_{j \in Y_{i}^{\mu_i}} \frac{Y_{i,j}}{|Y_{i}^{\mu_i}|} + 
    \frac{Y_{k,l} - \overline{Y_i^w}}{|Y_{k}^{\mu_k'}|} - C}{
    \sum_{i\in A} \frac{1}{|Y_{i}^{\mu_i'}|}} \\
    \theta_{t+1} &=& \theta_t +
    \frac{\frac{\theta_t|Y_{k}^{\mu_k'}|+|Y_{k}^{\mu_k}|(Y_{k,l} - \overline{Y_i^w}-\theta_t)}{|Y_{k}^{\mu_k'}||Y_{k}^{\mu_k}|}}{
    \sum_{i\in A} \frac{1}{|Y_{i}^{\mu_i'}|}} \\
    \theta_{t+1} &=& \theta_t +
    \frac{\frac{\theta_t+|Y_{k}^{\mu_k}|(Y_{k,l} - \overline{Y_i^w})}{|Y_{k}^{\mu_k'}||Y_{k}^{\mu_k}|}}{
    \sum_{i\in A} \frac{1}{|Y_{i}^{\mu_i'}|}} \\
\end{eqnarray*}
 
\begin{equation} \label{eq:ProjC:c}
    \theta >  j X_{i,j+1} - S_{i,j}
\end{equation}
This condition is sufficient to ensure an increasing $\theta$. 

When a row $k$, previously used until its $l$th element is removed:
Then we have:
\begin{eqnarray*} \label{eq:thetaPrimeA}
    \theta_{t+1} &=& \frac{\sum_{i\in A'} 
    \sum_{j \in Y_{i}^{\mu_i}} \frac{Y_{i,j}}{|Y_{i}^{\mu_i}|} + \overline{Y_k^l} - \overline{Y_k^l} - C}{
    \sum_{i\in A'} \frac{1}{|Y_{i}^{\mu_i'}|}} \\
    \theta_{t+1} &=& \theta_t + \frac{\frac{\theta_t}{|Y_k^l|} - \overline{Y_k^l}}{
    \sum_{i\in A'} \frac{1}{|Y_{i}^{\mu_i'}|}} \\
\end{eqnarray*}
This time, it is clear that if the sum of the values of the removed row is below $\theta$, 
then the row can be safely removed and the $\theta$ is increased.

\end{document}